\documentclass{ecai} 

\usepackage{latexsym}
\usepackage{amssymb}
\usepackage{amsmath}
\usepackage{amsthm}
\usepackage{booktabs}
\usepackage{enumitem}
\usepackage{graphicx}
\usepackage{color}
\usepackage{stmaryrd}

\newcommand{\sectionWithLabel}[1]{\section{#1}\label{sec:#1}}

\theoremstyle{remark}
\newtheorem*{remark}{Remark}
\newtheorem{theorem}{Theorem}

\newtheorem{example}[theorem]{Example}
\newtheorem{corollary}[theorem]{Corollary}

\newtheorem{definition}{Definition}

\newcommand{\BibTeX}{B\kern-.05em{\sc i\kern-.025em b}\kern-.08em\TeX}

\begin{document}

%%%%%%%%%%%%%%%%%%%%%%%%%%%%%%%%%%%%%%%%%%%%%%%%%%%%%%%%%%%%%%%%%%%%%%%%

\begin{frontmatter}

%%% Use this command to specify your submission number.
%%% In doubleblind mode, it will be printed on the first page.

\paperid{933} 

%%% Use this command to specify the title of your paper.

\title{Examining Independence in Ensemble Sentiment Analysis: A Study on the Limits of Large Language Models Using the Condorcet Jury Theorem}

%%% Use this combinations of commands to specify all authors of your 
%%% paper. Use \fnms{} and \snm{} to indicate everyone's first names 
%%% and surname. This will help the publisher with indexing the 
%%% proceedings. Please use a reasonable approximation in case your 
%%% name does not neatly split into "first names" and "surname".
%%% Specifying your ORCID digital identifier is optional. 
%%% Use the \thanks{} command to indicate one or more corresponding 
%%% authors and their email address(es). If so desired, you can specify
%%% author contributions using the \footnote{} command.

\author[A,B]{\fnms{Baptiste Lefort}~\snm{Lefort}\thanks{Corresponding Author. Email: baptiste.lefort@centralesupelec.fr}\footnote{Equal contribution.}}

\author[A,C]{\fnms{Eric Benhamou}~\snm{Benhamou}\footnotemark}

\author[A]{\fnms{Jean-Jacques Ohana}}
\author[A]{\fnms{Beatrice Guez}}
\author[A]{\fnms{David Saltiel}}
\author[A]{\fnms{Thomas Jacquot}} 

\address[A]{Ai for Alpha}
\address[B]{CentraleSupelec}
\address[C]{Dauphine - PSL}

\begin{abstract}
This paper explores the application of the Condorcet Jury theorem to the domain of sentiment analysis, specifically examining the performance of various large language models (LLMs) compared to simpler natural language processing (NLP) models. The theorem posits that a majority vote classifier should enhance predictive accuracy, provided that individual classifiers' decisions are independent. Our empirical study tests this theoretical framework by implementing a majority vote mechanism across different models, including advanced LLMs such as ChatGPT 4. Contrary to expectations, the results reveal only marginal improvements in performance when incorporating larger models, suggesting a lack of independence among them. This finding aligns with the hypothesis that despite their complexity, LLMs do not significantly outperform simpler models in reasoning tasks within sentiment analysis, showing the practical limits of model independence in the context of advanced NLP tasks.
\end{abstract}
\end{frontmatter}

%%%%%%%%%%%%%%%%%%%%%%%%%%%%%%%%%%%%%%%%%%%%%%%%%%%%%%%%%%%%%%%%%%%%%%%%

%%%%%%%%%%%%%%%%%%%%%%%%%%%%%%%%%%%%%%%%%%%%%%%%%%%%%%%%%%%%%%%%%%%%%%%%

\section{Introduction}
The integration of Natural Language Processing (NLP) into the financial sector has been pivotal in providing insights from textual data, evolving significantly with the release of Large Language models and generative models \citep{arner2015evolution, fatouros2023deepvar}. In particular, NLP has allowed to do sentiment analysis. The terminology sentiments analysis emphasizes that the method aims at extracting market sentiments or tones and then at offering predictive insights thanks to its analysis \citep{Tetlock2007GivingMarket,schumaker2009textual}.

Compared to other NLP applications, financial sentiment analysis faces unique challenges \citep{wankhade2022survey}. Financial narratives often involve complex, domain-specific terminologies and exhibit a multiplicity of sentiments tied to different entities, which can render general sentiment analysis tools ineffective \citep{loughran2011liability, poria2017review, sensoy2018evidential}. The complexity is further exacerbated by the nuances of financial news, which may be reactive rather than predictive, and the difficulty of integrating sentiment scores into practical investment strategies \citep{yuan2020target,iordache2022investigating}.

Historically, sentiments analysis relied on simple machine learning techniques and predefined word lists \citep{mejova2009sentiment}. Later on, with the introduction of more sophisticated NLP techniques, and in particular the deployment of models such as BERT and FinBERT, sentiment analysis has significantly refined its accuracy \citep{devlin2018bert, liu2021finbert}. With the emergence of Large Language Models (LLMs) like GPT, and particularly its variants tailored for conversational applications such as ChatGPT, a new paradigm in AI has been set forth \citep{brown2020language,roumeliotis2023chatgpt,openai2023gpt4}. These models offer enhanced capabilities in understanding and generating human-like text, presenting new opportunities to improve sentiment analysis within financial contexts. 

However, it is not clear how these new models differ from previous ones, motivating for a detailed and theoretically sound comparison with previous ones. In this paper, we focus on examining the independence of ensemble sentiment analysis conducted through LLMs, utilizing the Condorcet Theorem to investigate the limits and capabilities of these models in financial applications. The goal is to determine the extent to which these models can operate independently and to what degree their analysis can be considered reliable and unique in the context of complex sentiment tasks. Understanding this independence is critical, as it directly impacts the efficacy of models in real-world financial decision-making scenarios.

The paper is structured as follows: Section \ref{sec:Related works} reviews related works. Section \ref{sec:Contributions} outlines the primary contributions of our work. Section \ref{sec:Condorcet Jury theorem} presents the theoretical framework for the Condorcet Jury theorem for bagging on multiple class classifications. Section \ref{sec:Bagging experiment} presents a bagging experiment for doing sentiment analysis on various LLMs and highlights that there is no real improvement when using larger models.  Section \ref{sec:Discussion} discusses why the Condorcet Jury theorem does not hold in this setting, proving a lack of independence among these models and suggesting that more recent models do not perform better in reasoning tasks within sentiment analysis. Section \ref{sec:Conclusion} provides concluding remarks and future research directions.

\sectionWithLabel{Related works}
The exploration of ensemble methods in natural language processing, specifically through the lens of the Condorcet theorem, has garnered considerable attention in recent literature. This section discusses works closely related to our research and explains our contribution or similarities.

\begin{enumerate}
    \item \textbf{Ensemble Methods in Sentiment Analysis:} A significant portion of research in sentiment analysis has advocated for the use of ensemble techniques to improve classification accuracy. For example, \cite{kazmaier2022power} demonstrated that ensemble classifiers could robustly enhance sentiment detection across various social media platforms by leveraging multiple model predictions. Their work primarily focused on conventional machine learning models and did not explore the application of large language models (LLMs) in their ensemble. Our study extends this area by incorporating LLMs like GPT 3.5, GPT-4 and a finetuned version of GPT3.5 into the ensemble framework, testing the limits and benefits of bagging such advanced models in sentiment analysis.

    \item \textbf{Application of Condorcet Theorem in Machine Learning:} The theoretical underpinning of our study is based on the Condorcet Jury theorem, historically used in majority voting \citep{boland1989majority,austen1996information}. This theorem proves that a majority vote for a binary decision among independent classifiers leads to superior collective accuracy under certain assumptions \citep{mercier2019majority}. \cite{srivastava2022ensemble} used the Condorcet Jury theorem to explain why their ensemble deep neural networks to screen for Covid-19 and pneumonia from radiograph images were demonstrating improved diagnostic outcomes. In our case, we extend the proof of the Condorcet Jury Theorem to multi-class problems by employing the assumption of an Independent, Well-Trained, and Uniformly Biased (IWTUB) set of classifiers toward the correct alternative. Additionally, we apply the Condorcet Jury Theorem in a novel context, shifting from an image recognition problem to evaluating the independence of decisions among various large language models (LLMs) and simpler natural language processing (NLP) models within a sentiment analysis task.

    \item \textbf{Independence of Classifiers in Advanced NLP Models:} The assumption of independence among classifiers is critical in applying the Condorcet theorem effectively. \cite{galar2011overview} provide an overview of ensemble methods for binary classifiers in multi-class problems, highlighting challenges in classifier independence. They noted that in many practical applications classifiers often exhibit correlated errors, which can significantly diminish the efficacy of ensemble methods. Our findings are in line with these perspectives, revealing that the anticipated independence among various LLMs, including state-of-the-art models like ChatGPT 4, is not sufficiently pronounced to harness the full potential of the majority vote mechanism dictated by the Condorcet theorem.
    
    \item \textbf{Performance of LLMs Versus Simpler Models:} Contrary to several prior studies that highlighted the superiority of LLMs in complex reasoning tasks \cite{openai2023gpt4,Lopez-Lira2023CanModels,lefort2024can}, our empirical results indicate only marginal performance improvements when these models are included in an ensemble for sentiment analysis. 
    \cite{zhang2023sentiment,lefort2024small} also offer a critical perspective on the effectiveness of LLMs in sentiment analysis, questioning their practical superiority in sentiments analysis. Likewise, \cite{zhang2023instruct} observed experimentally that although LLMs consistently outperform Smaller Language Models (SLMs) in a few-shot learning context (\citep{brown2020language}), they struggle on more complex models. This is consistent with our findings that large language models (LLMs) fail to outperform simpler models such as FinBERT \citep{araci2019finbert} and DistilRoBERTa \citep{sanh2020distilbert} in sentiment analysis tasks, suggesting their decisions are not independent.
\end{enumerate}

\sectionWithLabel{Contributions}
This paper seeks to enrich the field of natural language processing (NLP) by leveraging the Condorcet Jury Theorem within the specific context of sentiment analysis. Sentiment analysis requires sophisticated multi-class classification approaches. The key contributions can be summarized as follows:

\begin{enumerate}
    \item \textbf{Theoretical Contribution:} We expand the Condorcet Jury Theorem, which to our knowledge has only been proven for binary decisions, to address the more complex scenario of multi-class classification. This extension relies on the new introduced concept of IWTUB set. Doing multi-class is crucial for sentiment analysis as sentiments typically involves categorizing texts into multiple classes beyond simple binary labels, such as positive, negative, and neutral or even potentially more granular sentiments.
    
    \item \textbf{Empirical Validation of Model Non Independence:} Using a majority classifier with multiple NLP models, including fine-tuned FinBERT, DistilRoBERTa, GPT-3.5, and GPT-4, we found that majority voting did not improve performance. This suggests a lack of independence among the models, as indicated by the IWTUB set's last condition. Unlike traditional statistical tests like Pearson's chi-squared, Spearman's rank correlation, and mutual information, which assess dependencies between dataset variables, the Condorcet Jury Theorem uses voter competence and independence to predict the accuracy of group decisions, focusing on decision-making processes rather than dataset characteristics.
    
    \item \textbf{Insights on LLMs:} Our results suggest a significant overlap in the decision-making processes for financial sentiment analysis of both compact and advanced models and a potential lack of reasoning for generative LLMs like GPT.
\end{enumerate}

The paper's intuition is detailed in figure \ref{fig:paper_intuition} for enhanced clarity on the overall methodological approach.

\section{Extension of Condorcet jury theorem to multi-class}\label{sec:Condorcet Jury theorem}

\begin{figure*}
    \centering
    \resizebox{0.9\textwidth}{!}{
    \includegraphics{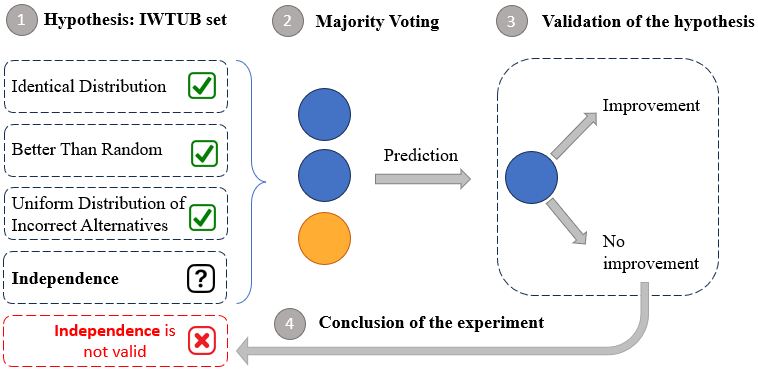}
    }
    \caption{Procedure diagram of models independence validity revocation in the Condorcet's theorem.}
    \label{fig:paper_intuition}
\end{figure*}

The Condorcet Jury theorem \citep{Condorcet1785}, was initially proven for a jury that should vote for a yes or no question. It states that if there is a majority preference for one option, then that option should be chosen as the overall winner in a collective decision-making process, provided that individual jury members are better at deciding than a blind guess and that there are independent \citep{list2001epistemic,morreau2021democracy}. Mathematically, if we change individual jury members into machine learning classifiers and assume that each classifier follows the same law, an ensemble majority voting bagging classifier should perform better than any individual classifiers. In this section, we will see how to extend these results to a multi class classification. We will need to assume that classifiers are somehow similar. This is the core assumption that will be challenged in our LLMs experiment.

\begin{definition}\label{def:IWTUB Set} \textbf{IWTUB Set:}  We say that a set of $n$ classifiers $\{C_1, C_2, \ldots, C_n\}$ is an IWTUB (Independent, Well-Trained and Uniformly Biased towards the correct Alternative) Set if all the classifiers satisfy the following four conditions:
\begin{enumerate}
    \item Identical Distribution: Classifiers have the same distribution.
    \item Better Than Random: Each classifier $C_i$ has an accuracy better than random guessing over the $k$ possible predicted labels, i.e., for each $C_i$ and for the true label $t$,  the probability $P(C_i(x) = t) $ is more than $\frac{1}{k}$. For binary classification scenarios, this threshold stands at 0.5. In our specific context, where there are three possible classes (positive, indecisive or negative), the threshold adjusts to one-third.
    \item Uniform Distribution Among Incorrect Alternatives: Each classifier $C_i$ has the same  probability which is lower than random guessing for choosing the incorrect alternatives, i.e., for each $C_i$ and for any incorrect label $y \neq t$,  the probability $P(C_i(x) = y)$ is less than $\frac{1}{k}$, 
    \item Independence: Classifiers make their predictions independently of each other.
\end{enumerate}
\end{definition}

Among all the assumptions, the last one (independence of models) is usually quite problematic as there are often correlation in the way models handle data. Under the assumption of IWTUB Set of classifiers, we can extend the Condorcet theorem to multi class classifications problems as follows.

\begin{theorem} \label{prop:Condorcet Theorem} \textbf{Condorcet Jury Theorem for Classifiers:}
For an IWTUB Set of classifiers ${C_1, C_2, \ldots, C_n}$, the majority vote classifier $C_{\text{bag}}$ exhibits higher accuracy than any individual classifier within the ensemble. When the number of classifiers tends to infinity, the majority classifier converges in probability to the true label. Conversely, if base classifiers perform worse than random guessing, as the number of classifiers tends to infinity, majority classifier will diverge in probability from the true label and hence be worse than the random classifier.
\end{theorem}

\begin{proof} The proof is an adaptation of the traditional two possible predicted labels presented for instance in \citep{sancho2022probability}. It relies on two arguments. The Khintchine weak law of large numbers is used to establish the convergence  in probability of the majority classifier applied to the set of our classifiers to the correct solution. Recall that a majority vote classifier is transitive \citep{kuncheva2006classifier,kuncheva2014majority} meaning that changing the order of classifiers does not change the outcome. Recall also that a majority vote classifier  is a fair function \cite{james2001majority}, meaning any true label plays same role and that applying a random permutation on the false labels does not change the result under the assumption of uniform distribution among incorrect alternative. \\

A classifier $C_i$ takes some inputs $X$ (in our NLP task some words) and provides an answer $y$ (in our case a numerical sentiment) indicating which label is the true label among possible answers that are integers ranging from $1$ to $k$. Because of fairness, there is a symmetry in true labels. Proving the result for one specific label will ensure that the result holds for any other true label. So without loss of generality, given some inputs $X$, we can assume that the true label is the largest label $k$ to make our computation simpler. \\

In this specific case, a classifier would perform better than random guessing if its probability to spot the true label is higher than random guessing. Because of the symmetry of the problem, random guessing probability is $1 / k$. Hence, if we define $a$  the constant, that we also  refer in the paper as the advantage of our classifier $C_i$ over randomness, the probability to claim the true label is:
\begin{equation}
P(C_i = k) = \frac{1}{k} + a
\end{equation}
with $ a > 0 $. Because our classifiers have equal probability of choosing the incorrect alternatives, the alternative probabilities for choosing any incorrect label $l \neq k$ is given by 
\begin{equation}
P(C_i = l)  = P(C_i \neq k) = \frac{1}{k} - \frac{a}{k-1}
\end{equation}
In other words, we will choose the incorrect label with a probability dminished from the random guess by the fraction of the advantage over the wrong labels in number \( k - 1\).

As alternative probabilities are positive, the advantage should satisfy $a \leq \frac{k-1}{k}$. Because of the assumption of independence and identical distribution of all our classifiers, each classifier has the same expectation which can be easily computed as the sum of possible outcomes times their corresponding probabilities, that can be split between the true label and the wrong labels.
\begin{eqnarray}
\mathbb{E}[ C_i ] &=&  k  P(C_i = k) + \sum_{j=1}^{k-1} j  P(C_i = j)\\
% &=&  1 + a k + \sum_{j=1}^{k-1} j \left(  \frac{1}{k} - \frac{a}{k-1} \right)  \\
&=& \frac{k+1}{2} + a \frac{k}{2}  \label{eq:expectation}
\end{eqnarray}

The result obtained above is very interesting. It says that the expectation of a classifier is equal to the average random guess among all possible labels, (the term $\frac{k+1}{2}$) plus the advantage part (the term $ a \frac{k}{2}$). \\

If we impose that classifiers give equal weights to wrong label when making a wrong prediction, a series of classifiers will have a majority vote result equal to $k$ if within the series of classifiers, the dominating class is the $k$ class. If we denote by $N_{a}( \{ C_1, ..., C_n \} )$ (or simply $N_{a}$) the number of occurrence of $a$ in the predicted labels of the series of classifiers $\{ C_1, ..., C_n \}$, our assumption of uniform distribution among error will imply that assuming the wrong label $j$ is equivalent to assuming the wrong label $1$
\begin{eqnarray}
N_{j}( \{ C_1, ..., C_n \} )  = & N_{1}( \{ C_1, ..., C_n \} ) & \text{for} \,\, j \neq k
\end{eqnarray}

Hence, the dominating class will be $k$ if and only if $N_{k}\left( \{ C_1, ..., C_n \} \right)  >  N_{j}( \{ C_1, ..., C_n \} )$ for $j \neq k$ or simply $N_{k}  > N_{1}$. We can notice that we can directly infer from the sum of predicted labels whether the dominating class is $k$ or not as a classifier will have value of $k$ \(N_{k}\) times and otherwise all other labels \(N_{1} \) times. This gives us:
\begin{eqnarray}
\sum_{i=1}^n C_i & =  & k N_{k} + \frac{(k-1) k}{2} N_{1}  \\
&= &  \frac{ k (k+1)}{2} N_{1} + k \left( N_{k} -  N_{1} \right)
\end{eqnarray}

Hence,
\begin{eqnarray}
N_{k}  > N_{1} &\Leftrightarrow  &  \sum_{i=1}^n C_i > \frac{ k (k+1)}{2} N_{1}\\ &\Leftrightarrow &\frac{\sum_{i=1}^n C_i}{n} > \frac{ k+1}{2} 
\end{eqnarray}

\noindent where in the last part, we have used that $k N_{1} + \left( N_{k}-  N_{1} \right) = n$. In other terms, we get
\begin{equation}
\mathbb{P} \left[ \text{Majority is correct} \right]  =  \mathbb{P}\left[ \frac{\sum_{i=1}^n C_i}{n}  > \frac{k+1}{2} \right] 
\end{equation}

The Khintchine weak law of large numbers \citep{khintchine1929sur} states that the empirical average $ \frac{\sum_{i=1}^n C_i}{n} $ converges in probability to the expectation of the identically distributed classifiers. In other words, for any $\varepsilon_1, \varepsilon_2 >0 $, there exists $N$ large enough such that for every $n \geq N$, $\mathbb{P}\left( \left| \frac{\sum_{i=1}^n C_i}{n} - \mathbb{E}[ C_i ] \right| < \varepsilon_1 \right) \geq 1 - \varepsilon_2$. 

In particular for $\varepsilon_1 = a \frac{k}{2}$, we have for every $n \geq N$
\begin{eqnarray}
& & \mathbb{P}\left[ \text{Majority is Correct}\right] \\
& = & \mathbb{P}\left(\frac{\sum_{i=1}^n C_i}{n}  > \frac{k+1}{2} \right)\\
& = & \mathbb{P}\left( \frac{\sum_{i=1}^n C_i}{n} - \frac{k+1}{2} - a \frac{k}{2}   > - a \frac{k}{2}  \right) \\
& \geq & \mathbb{P}\left( a \frac{k}{2} > \frac{\sum_{i=1}^n C_i}{n} - \frac{k+1}{2} - a \frac{k}{2}   > - a \frac{k}{2}  \right) \\
& \geq & \mathbb{P}\left( \left| \frac{\sum_{i=1}^n C_i}{n} - \mathbb{E}[ C_i ] \right| <  \varepsilon_1  \right) \label{eq:constant_a} \\
% & \geq & \mathbb{P}\left( \left| \frac{\sum_{i=1}^n C_i}{n} - \mathbb{E}[ C_i ] \right| <  \varepsilon_1  \right) \\
& \geq & 1 - \varepsilon_2
\end{eqnarray}
where in equation \eqref{eq:constant_a}, 
we have use equation \eqref{eq:expectation} to introduce the expectation of any classifier $\mathbb{E}[ C_i ]$.

The last equation states that for any $\varepsilon_2 > 0$, there exists $N$ such that for every $n \geq N$, we have 
\begin{eqnarray}
\mathbb{P}\left[ \text{Majority is Correct}\right] \geq &1 - \varepsilon_2 \label{eq:majority}
\end{eqnarray}

As equation \eqref{eq:majority} holds for any \(\varepsilon_2 \), this effectively proves that the probability of taking the true label would converge to $1$. Hence, we have effectively proved that majority vote function would choose the true label with a probability that converges to $1$ as the number of classifiers becomes large and tends to infinity.
Last but not least, it is worth noticing that we have used in our proof that the advantage $a$ is constant among all classifiers in equation \eqref{eq:constant_a}. \\

A few remarks apply. The result can generalised to non constant advantages functions and this is the subject of theorem \ref{theo:Advantage condition}. The  Khintchine weak law of large numbers is justified as the variance of the average of predicted values of our classifiers is finite. We could have also be tempted to use the Kolmogorov strong law of large numbers to establish that the convergence of the average of predicted values of our classifiers  converges almost surely to the expected value of our i.i.d. classifiers \citep{kolmogorov1933foundations}. However, in pure generality, since the convergence is not necessarily uniform on the set where it holds, it does not imply that with probability $1$, we have that for any $\varepsilon > 0$, the inequality $\left| \frac{\sum_{i=1}^n C_i}{n} - \mathbb{E}[ C_i ] \right| < \varepsilon$ holds for all large enough n, which would be problematic. As a matter of fact, it turns out that the latter does hold as again the variance of  variance of the average of predicted values of our classifiers is finite, which somehow  circles us back to to the  Khintchine weak law of large numbers.
\end{proof}

\begin{corollary}
 Additionally if the independence assumption does not hold, the majority classifier $C_{\text{bag}}$ should not perform better than the best classifier.
 \end{corollary}

\begin{proof} straightforward consequence of theorem \ref{prop:Condorcet Theorem}
\end{proof}

\subsection{Non Uniform Classifiers}
Up to this point, our analysis has assumed that all classifiers are uniformly distributed and equally accurate. We will now relax these assumptions to consider classifiers that vary in performance but are still more effective than random guessing. To explore how these non-uniform classifiers compare to random chance, we refine the concept of \textit{Classifier Advantage}, which measures the extent to which a classifier outperforms random guessing.

\begin{definition}\textbf{Classifier Advantage:}\label{def:classifier_advantage} Let $t$ be the true label. For a Uniformly Biased Towards the Correct Alternative (UBTCA) classifier $C_i$, we define its positive advantage $a_i$ as its advantage or edge over random guessing:
\begin{equation}
a_i = P(C_i = t) - \frac{1}{k}
\end{equation}
\end{definition}

\begin{remark}
To have a probability above random guessing, the advantage $ a_i $ should be positive. Because our classifier has equal probability of choosing the incorrect alternatives, the alternative probability for $l \neq t$ is given by $
P(C_i = l)  = P(C_i \neq t) = \frac{1}{k} - \frac{a}{k-1} $. This leads in particular to an upper bound for the advantage to ensure that probabilities are positive: $ a \leq \frac{k-1}{k}$.
\end{remark}

We can now prove the condition of these classifiers relative to random guessing, establishing that the sum of advantages should grows significantly faster than $\sqrt{n}$, or equivalently that the advantage term should converge to zero slower than the inverse of the square root of the number of classifiers.\\

\begin{theorem}\label{theo:Advantage condition}
\textbf{Condorcet Jury Theorem for Non-Equal Classifiers:} A majority vote classifier for non-equal classifiers, biased towards the correct alternative, will converge to the true solution in probability as their number increases, provided the partial sum of advantages grows faster than the square root of their count:
\begin{equation}\label{eq:Advantage condition}
\frac{\sum_{i=1}^n a_i}{\sqrt{n}} \underset{ n \to \infty }{\to} \infty.
\end{equation}
\end{theorem}

\begin{proof} 

Consider a sequence of independent and identically distributed (i.i.d.) random variables \(C_1, C_2, \ldots\), which we refer to as classifiers. Each classifier \(C_i\) can produce multiple potential outcomes, but we assume a uniform distribution among errors. Using the symmetry of our problem, without loss of generality we can assume that the true label among $k$ possible ones is the largest value to make calculations easier. The assumption of equal distribution among errors transforms the initial multinomial setting into a binomial case, where the classifier either predicts the true label $t=k$ or the wrong labels $j \neq k$. \\

In a balanced binomial scenario, the probability of success (correct classification) and failure (incorrect classification) is traditionally set to \(1/2\). However, given our setup where classifiers do not distinguish between the types of incorrect labels and the error is uniformly distributed across \(k\) possible outcomes, the probability of an correct classification for each trial is \(1/k\). Denoting as in traditional statistical textbooks the partial sum \(S_n = \sum_{i=1}^{n} C_i\) of the outcomes of our classifiers \(C_i\), the sequence \(\sqrt{n}\left(\frac{S_n}{n} - \mathbb{E}\left[ \frac{S_n}{n} \right] \right)\), converges in distribution to a normal distribution \(\mathcal{N}(0, \sigma^2)\) as \(n \to \infty\), where \(\sigma^2 \) is the variance of our i.i.d.s classifiers\\

In our specific settings, classifiers are not any more identically distributed as their advantage function $a_i$ is not constant. However, we can still apply the Lyapunov central limit theorem (see \citep{billingsley1986probability} p. 371), or the central limit theorem for triangular
arrays (see \citep{kallenberg2002foundations} Theorem 5.15) to see that the central limit theorem still holds as the individual probabilities of true classification $p_i = \frac{1}{k} + a_i$ converges or equivalently that the advantage functions $ a_i$  converges and consequently that the individual variance of our classifier converges to a finite value which we denote $\Sigma$ and we will determine later.\\
Hence, in our case, we have
\begin{eqnarray} \label{eq:CT}
\sqrt{n} \left(\frac{S_n}{n} - \mathbb{E}\left[ \frac{S_n}{n} \right] \right) \xrightarrow[n \to \infty]{} \mathcal{N} (0, \Sigma )
\end{eqnarray}

Let us refine our computation. For each classifier, we can compute their expectation and variance as follows
\begin{eqnarray}
\mathbb{E}[ C_i ] & = & k  P(C_i = k) + \sum_{j=1}^{k-1} j  P(C_i = j) \\
% & = &  k \left( \frac{1}{k} + a_i \right) + \sum_{j=1}^{k-1} j \left(  \frac{1}{k} - \frac{a_i}{k-1} \right)   \\
& = &  \frac{k+1}{2} + a_i \frac{k}{2} \\
\mathbb{E}[ C_i^2 ] &= & k^2  P(C_i = k) + \sum_{j=1}^{k-1} j^2  P(C_i = j) \\
 &=&  \frac{(k+1)(2k+1)}{6} + a_i \frac{k (8k-1)}{6} \\
\mathbb{V}[ C_i] &= &\mathbb{E}[ C_i^2 ]  - \left(\mathbb{E}[ C_i ] \right)^2 \\ 
& = &\frac{k^2-1}{12} + a_i \frac{k (5 k - 4) }{6} - a_i^2 \frac{k^2}{4}
\end{eqnarray}
The latter equation providing individual variance $\mathbb{V}[ C_i]$ can be seen as a quadratic function of $a_i$. It attains its maximum when its derivative given by \( \frac{k (5 k - 4) }{6} - 2  a_i \frac{k^2}{4} \) is null. This occurs when \(  a_i = \frac{5 k -4}{3k} \). In particular, for an advantage function $a_i$ ranging from $0$ to $\frac{k-1}{k}$, 
the individual variance $\mathbb{V}[ C_i]$ is bounded. In addition, if we assume that our advantage function $a_i$ converges to a finite value $a_\infty$ as $i$ tends to infinity, we have that the individual variance would also converge to a finite quantity $\Sigma$:
\begin{eqnarray}
\mathbb{V}[ C_i]  &\xrightarrow[i \to \infty]{} & \frac{k^2-1}{12} + a_\infty \frac{k (5 k - 4) }{6} - a_\infty^2 \frac{k^2}{4} \end{eqnarray}
Hence the asymptotic variance \( \Sigma\) is given by:
\begin{equation}
\Sigma = \frac{k^2-1}{12} + a_\infty \frac{k (5 k - 4) }{6} - a_\infty^2 \frac{k^2}{4}
\end{equation}

Using the linearity of the expectation, we have
\begin{eqnarray}
\mathbb{E}\left[ \frac{S_n}{n} \right] & = &  \frac{\sum_{i=1}^{n} \mathbb{E}\left[  C_i \right] }{n}   =  \frac{k+1}{2} +  \frac{k}{2n}  \sum_{i=1}^{n} a_i  \\
% \mathbb{V}\left[ \frac{S_n}{n} \right] &=  &  \frac{\sum_{i=1}^{n} \mathbb{V}\left[  C_i \right] }{n}   =   \frac{k^2-1}{12} + \frac{k (5 k - 4) }{6n} \sum_{i=1}^{n} a_i   - \frac{k^2}{4n} \sum_{i=1}^{n} a_i^2 
\end{eqnarray}

We have now reached a point where we can establish the proof of our theorem. The majority function will accurately indicate the true label if and only if:
\begin{eqnarray}
& &
\frac{S_n}{n}  > \frac{k+1}{2} \Leftrightarrow  \sqrt{n}  \left( \frac{S_n}{n} - \frac{k+1}{2} \right) > 0 \hspace{2cm}\\
& & \hspace{1cm} \Leftrightarrow   \sqrt{n} \left(\frac{S_n}{n} - \mathbb{E}\left[ \frac{S_n}{n} \right] \right) + \frac{k}{2 \sqrt{n} }  \sum_{i=1}^{n} a_i  > 0 
\end{eqnarray}
Using the asymptotic distribution provided in \eqref{eq:CT}, and for large value of $n$, this will be equivalent to
\begin{eqnarray}\label{eq:condition}
 \mathcal{N} (0, \Sigma ) + \frac{k}{2 \sqrt{n} }  \sum_{i=1}^{n} a_i  > 0
\end{eqnarray}
Hence,  should the following series diverge,
\begin{eqnarray}
\frac{\sum_{i=1}^n a_i}{\sqrt{n}} & \underset{ n \to \infty }{\to} & \infty,
\end{eqnarray}
thanks to dominated convergence, the inequality \eqref{eq:condition} will hold with a probability  that converges to 1 as $n$ tends to infinity as:
\begin{eqnarray}\label{eq:inequality2}
& & \lim_{n \to \infty}  \mathbb{P}\left[ \mathcal{N} (0, \Sigma ) + \frac{k}{2 \sqrt{n} }  \sum_{i=1}^{n} a_i  > 0 \right]\\
& =  & \mathbb{P}\left[ \mathcal{N} (0, \Sigma ) > -  \lim_{n \to \infty} \frac{k}{2 \sqrt{n} }  \sum_{i=1}^{n} a_i   \right]  \\
 % & =  & \mathbb{P}\left[ \mathcal{N} (0, \Sigma ) > - \infty \right] \\
 & = & 1
\end{eqnarray}

This concludes the first proof stating that the divergence of the series ensures the correctness of the majority with probability converging to one with the growing number of classifiers $n$. Conversely, should the ratio of the sum of advantages to the square root of their total number approach zero as $n$ approaches infinity, the probability that the majority vote function identifies the correct outcome will converge to one half, given the symmetry of the asymptotic normal distribution $\mathcal{N} (0, \Sigma )$ around the origin. In other terms, the majority vote function will find the correct result like a random guessing, which concludes the proof.
\end{proof}

\begin{example}\textbf{Proper advantage functions}
We can provide multiple examples of advantage series meeting the divergence condition given by condition \eqref{eq:Advantage condition}. 
\begin{itemize}
    \item Logically, the simplest one is a constant strictly positive advantage: $ a_i = \lambda > 0$. This is easy to check as 
    \begin{eqnarray}
        \frac{\sum_{i=1}^n a_i}{\sqrt{n}} \sim   \frac{\lambda n \sqrt{n}}{2} & \underset{ n \to \infty }{\to} & \infty
    \end{eqnarray}
    \item Moreover, the advantage can have a logarithmic decrease to zero: $ a_i = \frac{1}{\log{i}}$. Indeed, using equivalent, we can check that the series properly diverges to infinity as follows:
    \begin{eqnarray}
        \frac{\sum_{i=1}^n a_i}{\sqrt{n}} \sim \frac{\sqrt{n}}{\log{n}} & \underset{ n \to \infty }{\to} & \infty
    \end{eqnarray}    
    
    \item  In the case of a polynomial decrease, it should be lower than the square root inverse as follows: $ a_i = \frac{1}{i ^{\alpha }}$ with $\alpha < 1/2$. For instance $ a_i = \frac{1}{ \sqrt[3]{i}}$ will work. In the general case, we can use the Riemann condition (\citep{apostol1974mathematical}) to check the divergence as
    \begin{eqnarray}
        \frac{\sum_{i=1}^n a_i}{\sqrt{n}} \sim n^{1/2-\alpha}  & \underset{ n \to \infty }{\to} & \infty
    \end{eqnarray}
\end{itemize}
\end{example}

\sectionWithLabel{Bagging experiment}
In this section, we introduce the results of the experiments made for comparing the state-of-the-art LLMs in financial sentiment classification: FinBERT \citep{araci2019finbert}, DistilRoBERTa \citep{sanh2020distilbert} and GPT-4. We detail the framework and the individual abilities of each models. This experimental section enables us to validate the non-applicability of the Condorcet's theorem in this context and subsequently to question the assumption of models independence.

\begin{figure*}
    \centering
    \resizebox{0.85 \textwidth}{!}{
    \includegraphics{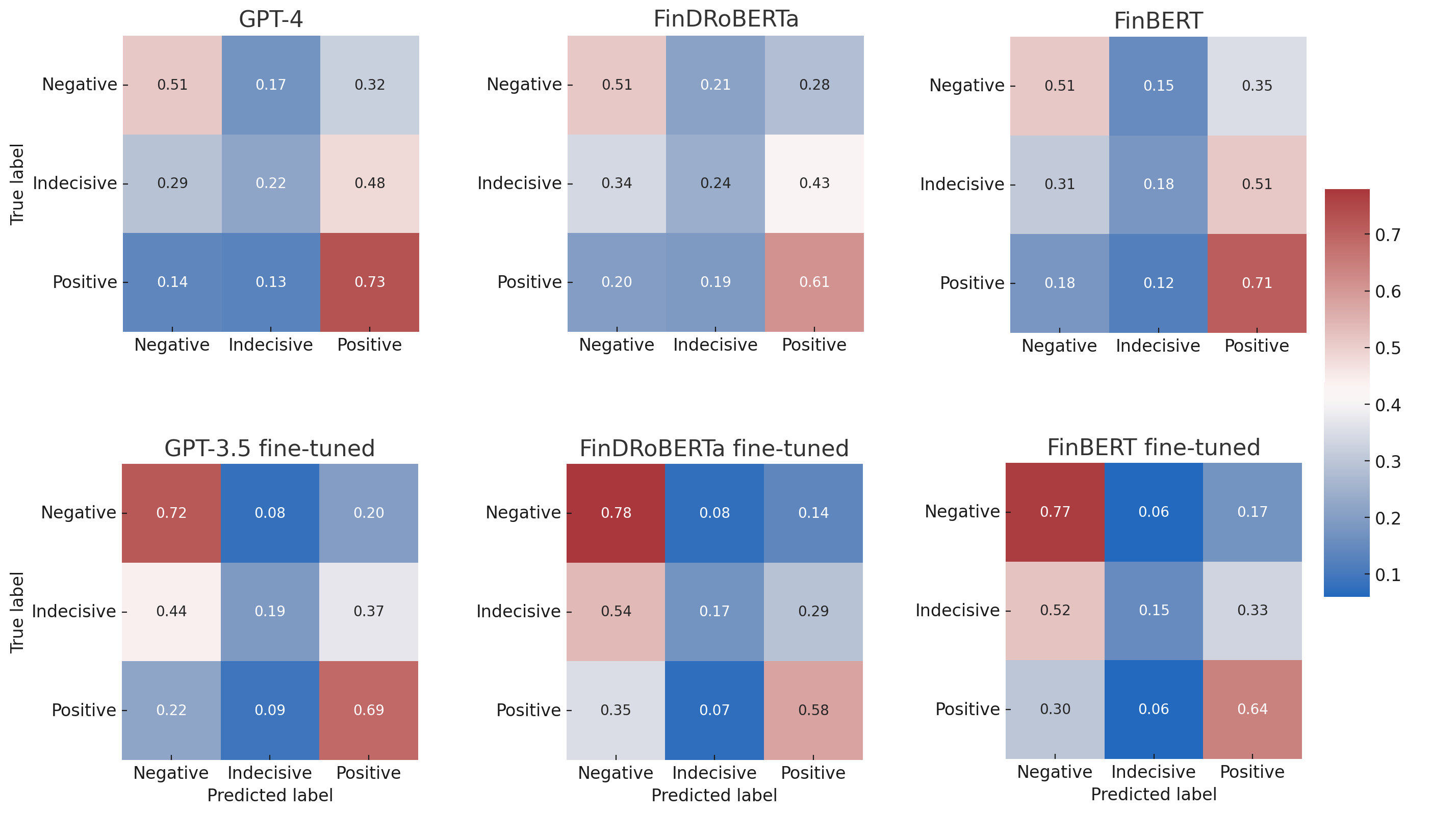}
    }
    \caption{Confusion matrices for each assessed models}
    \label{fig:confusion_matrices}
\end{figure*}

\subsection{Experimental framework}
We use a proprietary dataset of headlines generated from Bloomberg news called market wraps that span from 2010 to 2024 and is known to be of high quality as it is scrutinized by market professionals \citep{alzazah2020recent}. These headlines are a consolidated summary of daily financial news, done by human journalists specialized in finance, highly regarded and extensively followed by professionals within the financial sector. The database has about 65,000 rows representing on average 15 major headlines per day over 15 years. For each headlines, the database comes with the corresponding next day returns of major equities markets, hence providing a way to systematically provides a news sentiment that is predictive of future impact by definition. We can therefore validate if a model is able to predict effectively a predictive news sentiment. This dataset is containing high-quality data and is closely followed by the market players \citep{garcia2022detection}. The labels are balanced as detailed in table \ref{tab:data_distribution}. All these arguments make this dataset valuable for fine-tuning and accurately assessing the performance of the considered LLMs on the financial classification task.

\begin{table}[!htbp]
    \centering
    \caption{Number of instance in each class with percentage}
    \label{tab:data_distribution}
    \begin{tabular}{ccc}
        \toprule
        \bfseries Negative  & \bfseries Neutral  & \bfseries Positive \\
        \midrule
        19254 (\textit{31\%}) & 16202 (\textit{27\%}) & 25795 (\textit{42\%}) \\
        \bottomrule
    \end{tabular}
\end{table}

The dataset we present, as detailed in Table \ref{tab:dataset_sample}, is specifically designed to provide insights into economic events through news headlines, each linked to precise dates and associated sentiment classifications. This arrangement is crucial for analyzing how sentiments expressed in financial headlines correlate with market events and conditions, offering a robust foundation for temporal trend analysis over an extensive period from 2010 to 2024. This allows for longitudinal studies examining the evolution of sentiment in response to global economic, political, and social changes. 

Each row of the dataset contains a sentiment label—Positive, Negative, or Indecisive—systematically assigned by an algorithm, enhancing its utility for financial forecasting. The robustness of this dataset, especially compared to public datasets such as Financial PhraseBank \citep{PhraseBank} and StockNet Dataset \citep{StockNet}, stems from its focus on actionable investment decisions without the noise often introduced by social media content or human annotation, which can be less predictive or non systematic and subjective. In particular, the direct linkage of each sentiment classification to subsequent market performance in this private dataset ensures that the labels reflect predictive market events, as discussed in the literature \citep{sezer2020financial,ashtiani2023news,xing2018natural}.

The primary objective of classifying sentiment in financial headlines is to generate predictive signals for the market. In our analysis, if a Large Language Model (LLM) can accurately classify the sentiment of all stocks mentioned in the headlines, it would yield a valuable trading signal. This capability underscores the dataset's effectiveness in enhancing market outcome predictions.

\begin{table}[!htbp]
\centering
\caption{Dataset Sample}
\label{tab:dataset_sample}
\resizebox{\columnwidth}{!}{%
    \begin{tabular}{llc}
        \toprule
        \bfseries Date & \bfseries Headline & \bfseries Sentiment \\
        \midrule
        2010-01-04 & Oil Prices Surge Above \$81 a Barrel Due to U.S. Cold Weather & Positive \\
        \ldots & \ldots & \ldots \\
        2014-07-02 & Euro and Pound Gain on Hawkish Central Bank Policies & Negative \\
        \ldots & \ldots & \ldots \\
        2020-11-06 & S\&P 500 Index Sees First Retreat in a Week & Indecisive \\
        \ldots & \ldots & \ldots \\
        2024-01-30 & General Motors Beats Q4 Expectations, Expects Profit Growth & Positive \\
        \bottomrule
    \end{tabular}
}
\end{table}

\subsection{Models assessment}
At this step, we examine the individual performances of the state-of-the-art models. They have both different architecture and different number of parameters. However they are considered as comparable for this classification task \cite{araci2019finbert, openai2023gpt4}.

\begin{table}[!htbp]
\centering
\caption{Model Performances Before Fine-Tuning}
\label{tab:perf_before_finetuning}
% \resizebox{\columnwidth}{!}{%
    \begin{tabular}{lccc}
        \toprule
        \bfseries Model & \bfseries F-score $\downarrow$ & \bfseries Precision & \bfseries Recall \\
        \midrule
        \bfseries GPT-4 & \bfseries 0.47 & \bfseries 0.48 & \bfseries 0.49 \\
        GPT-3.5 & 0.45 & 0.47 & 0.47 \\
        DistilRoBERTa & 0.44 & 0.45 & 0.45 \\
        FinBERT & 0.44 & 0.46 & 0.46 \\
        Random  Classifier  & 0.33 & 0.33 & 0.33 \\
        \bottomrule
    \end{tabular}
% }
\end{table}
Table \ref{tab:perf_before_finetuning} illustrates that generative models such as those in the GPT series do not outperform encoder-only models based on the BERT architecture, despite GPT's substantially larger number of parameters. This finding challenges the assumed predictive superiority of GPT models over more compact architectures like BERT. However, it is evident that all models significantly surpass the random benchmark, indicating their capacity to learn from the labeled data provided.

The extensive Bloomberg dataset utilized in this study facilitates precise fine-tuning of these models. We divided the dataset into two segments: one for training the models and the other for evaluating their performance, as shown in Table \ref{tab:perf_after_finetuning}, which reports the models' results on the test set. It is noteworthy that at the time of writing this paper, GPT-4 could not be fine-tuned, but we were able to fine-tune the latest version of GPT-3.5. Given the similar performance of the non-fine-tuned models, we anticipate comparable results when GPT-4 becomes available for fine-tuning.

The fine-tuned models demonstrate no significant differences in performance among themselves; all exhibit uniform improvements. This indicates that the fine-tuning process equally enhances each model, without any single model demonstrating distinct superiority. Thus, while individual models are equally enhanced, none emerges as distinctly superior.

\begin{table}[!htbp]
    \centering
    \caption{Fine-Tuned Model Performances}
    \label{tab:perf_after_finetuning}
    % \resizebox{\columnwidth}{!}{%
    \begin{tabular}{lccc}
        \toprule
        \bfseries Model & \bfseries F-score $\downarrow$ & \bfseries Precision & \bfseries Recall \\
        \midrule
        \bfseries SFT GPT-3.5 & \bfseries 0.51 &
        \bfseries 0.54 & 
        \bfseries 0.53 \\
        SFT DistilRoBERTa & 0.49 & 0.53 & 0.51 \\
        SFT FinBERT & 0.50 & 0.54 & 0.52 \\
         Random Classifier & 0.33 & 0.33 & 0.33 \\
        \bottomrule
    \end{tabular}
    % }
\end{table}

Based on these performances, we apply an ensemble strategy. Since all the models have comparable performances we could expect a performance improvement if the Condorcet's hypothesis hold. Below is the list of the two different bagging configurations, including Supervised Fine-Tuned (SFT) models.
\begin{itemize}
    \item Bagging 1: SFT GPT-3.5 + SFT DistilRoBERTa + SFT FinBERT
    \item Bagging 2: All the models SFT and No SFT
\end{itemize}

\begin{table}[!htbp]
    \centering
    \caption{Ensemble method performance compared with the best solo performer model SFT GPT-3.5. The difference of metric between the Bagging methods and the SFT GPT-3.5 model is provided.}
    \label{tab:ensemble_performance}
    \begin{tabular}{lccc}
        \toprule
        \bfseries Model & \bfseries $\Delta$ F-score $\downarrow$ & \bfseries $\Delta$ Precision  & \bfseries $\Delta$ Recall \\
        \midrule
        % Bagging 2 & 0.52 & 0.52 & 0.53 \\
        Bagging 2 vs SFT GPT-3.5 & 0.01 & 0.00 & -0.01 \\
        Bagging 1 vs SFT GPT-3.5 & 0.00 & 0.01 & -0.01 \\
        % \bfseries Random & \bfseries 0.33 & \bfseries 0.33 & \bfseries 0.33 \\
        \bottomrule
    \end{tabular}
\end{table}

Table \ref{tab:ensemble_performance} provides that no significant improvement is made with a bagging method. The ensemble classifier could not improve the overall performance.

\sectionWithLabel{Discussion}
Our experiments have confirmed that three of the four critical hypotheses necessary for applying Condorcet's theorem hold, namely.

\begin{enumerate}
\item \textbf{Identical Distribution} The confusion matrices in Figure \ref{fig:confusion_matrices} demonstrate that the classification performance distributions across the models are consistent. Both fine-tuned and non-fine-tuned models exhibit comparable results across different classes, with the probability distributions of each LLM being markedly similar.

\item \textbf{Better than random} As evidenced by the confusion matrices in Figure \ref{fig:confusion_matrices} and performance data in Tables \ref{tab:perf_before_finetuning} and \ref{tab:perf_after_finetuning}, all LLMs in our study perform significantly better than a random classifier.

\item \textbf{Uniform distribution among incorrect alternatives} The confusion matrices in Figure \ref{fig:confusion_matrices} indicate that the error distribution among incorrect choices is uniform across models. The similarity in error distribution and the high correlation in the instances where models fail further support this hypothesis.
\end{enumerate}

Since the three hypothesis hold and are confirmed by the experiments, the last condition of IWTUB set can not hold, namely the independence hypothesis of the LLMs models used in our bagging experiment. In particular, this indicates that GPT models cannot be considered as independent or different models from simpler models like FinBERT and DistilRoBERTa. In particular, this suggests that the reasoning capacity of GPT models may not be very different  form the smaller models and that they lack the capacity to reason on complex financial tasks like financial sentiment analysis. 

\sectionWithLabel{Conclusion}
This paper extends the Condorcet Jury Theorem, previously limited to binary decisions, to multi-class classification scenarios. This advancement incorporates the new concept of IWTUB set, introduced to facilitate the extension of the theorem to multi-class classification systems. Our empirical investigation of the applicability of a majority vote classifier revealed that despite the advanced capabilities of LLMs, such as GPT-3.5 or 4, there are only marginal enhancements in predictive accuracy when combined with smaller models. 
This outcome suggests a significant overlap in the decision-making processes of these models, underscoring their limited independence, which undermines the effectiveness of the majority voting mechanism from the Condorcet Jury Theorem and highlights the lack of reasoning abilities of LLMs for complex NLP task like sentiment analysis.  The use of the Condorcet Jury theorem is very complementary to traditional statistical tests for independence, such as Pearson's chi-squared test, Spearman's rank correlation, and mutual information, that only focus on establishing the absence of a relationship between two variables in a dataset as the Condorcet Jury theorem leverages the assumption of voter competence and independence to predict the likelihood of correct decision-making in group settings. 
Future research should complement this approach by investigating other methods, like statistical ones to assess the non-independence among classifiers and exploring scenarios where large language models (LLMs) are effectively implemented in ensemble settings.
%\Huge{NO MORE THAN 7 PAGES}

\clearpage
\bibliography{main}

\end{document}